\documentclass{article}
\usepackage{ijcai17}

\usepackage{times}

\usepackage{algorithmic}
\usepackage{algorithm}
\usepackage{graphicx}
\usepackage{amsmath,amssymb}
\usepackage{tikz}
\usepackage{arydshln}
\usepackage{amsfonts}

\newtheorem{definition}{Definition}
\newtheorem{theorem}{Theorem}

\newenvironment{proof}{\noindent{\bf Proof.}}{\hfill$\Box$}

\title{The MacGyver Test - A Framework for Evaluating Machine Resourcefulness and Creative Problem Solving}
\author{\textbf{Vasanth Sarathy} \qquad \textbf{Matthias Scheutz} \\
         Tufts University \\ Medford, MA, USA}

\begin{document}

\maketitle

\begin{abstract}
  Current measures of machine intelligence are either difficult to
  evaluate or lack the ability to test a robot's problem-solving
  capacity in open worlds. We propose a novel evaluation framework
  based on the formal notion of {\em MacGyver Test} which provides a
  practical way for assessing the resilience and resourcefulness of
  artificial agents.
\end{abstract}

\section{Introduction}

Consider a situation when your only suit is covered in lint and you do
not own a lint remover. Being resourceful, you reason that a roll of
duct tape might be a good substitute. You then solve the problem of
lint removal by peeling a full turn's worth of tape and re-attaching
it backwards onto the roll to expose the sticky side all around the
roll.  By rolling it over your suit, you can now pick up all the
lint. This type of everyday creativity and resourcefulness is a
hallmark of human intelligence and best embodied in the 1980s
television series {\em MacGyver} which featured a clever secret
service agent who used common objects around him like paper clips and
rubber bands in inventive ways to escape difficult life-or-death
situations.\footnote{As a society, we place a high value on our human
ability to solve novel problems and remain resilient while doing
so. Beyond the media, our patent system and peer-reviewed publication
systems are additional examples of us rewarding creative problem
solving and elegance of solution.}

Yet, current proposals for tests of machine intelligence do not
measure abilities like resourcefulness or creativity, even though
this is exactly what is needed for artificial agents such as
space-exploration robots, search-and-rescue agents, or even home and
elder-care helpers to be more robust, resilient, and ultimately
autonomous.

In this paper we thus propose an evaluation framework for machine
intelligence and capability consisting of practical tests for
inventiveness, resourcefulness, and resilience.  Specifically, we
introduce the notion of {\em MacGyver Test} (MT) as a practical
alternative to the Turing Test intended to advance research.

%We do not claim that the proposed test is a complete and sufficient test for all manner of machine intelligence. We do not even require that passing the proposed test requires matching human performance. Instead, the proposed test has a more modest purpose: to serve as a framework with which to decide whether a machine and its species can be completely autonomous and survive in the real world. The MacGyver test, we hope, will serve to uncover the sort architectural framework that is highly adaptable and capable of long-term evolution of a species.   

\section{Background: Turing Test and its Progeny}

Alan Turing asked whether machines could produce observable behavior
(e.g., natural language) that we (humans) would say required thought
in people \cite{Turing1950}. He suggested that if an interrogator was
unable to tell, after having a long free-flowing conversation with a
machine whether she was dealing with a machine or a person, then we
can conclude that the machine was ``thinking''. Turing did not intend
for this to be a test, but rather a prediction of sorts \cite{cooper2013alan}. Nevertheless, since Turing, others have developed tests for
machine intelligence that were variations of the so-called Turing Test
to address a common criticism that it was easy to deceive the
interrogator.

Levesque et al. designed a reading comprehension test, entitled the
{\em Winograd Schema Challenge}, in which the agent is presented a
question having some ambiguity in the referent of a pronoun or
possessive adjective. The question asks to determine the referent of
this ambiguous pronoun or possessive adjective, by selecting one of
two choices \cite{Levesque2012}. Feigenbaum proposed a variation of
the {\em Turing Test} in which a machine can be tested against a team
of subject matter specialists through natural language
conversation \cite{Feigenbaum2003}. Other tests attempted to study a
machine's ability to produce creative artifacts and solve novel
problems \cite{Boden2010,Bringsjord2001,Bringsjord2016,Riedl2014}.

Extending capabilities beyond linguistic and creative, Harnad's {\em
Total Turing Test} (T3) suggested that the range of capabilities must
be expanded to a full set of robotic capacities found in embodied
systems \cite{Harnad1991}. Schweizer extended the T3 to incorporate
species evolution and development over time and proposed the {\em
Truly Total Turing Test} (T4) to test not only individual cognitive
systems but whether as a species the candidate cognitive architecture
in question is capable of long-term evolutionary
achievement \cite{Schweizer2012}.

Finding that the {\em Turing Test} and its above-mentioned variants
were not helping guide research and development, many proposed a
task-based approach. Specific task-based goals were designed couched
as toy problems that were representative of a real-world
task \cite{Cohen2005}. The research communities benefited greatly from
this approach and focused their efforts towards specific machine
capabilities like object recognition, automatic scheduling and
planning, scene understanding, localization and mapping, and even
game-playing. Many public competitions and challenges emerged that
tested the machine's performance in applying these capabilities --
from image recognition contests and machine learning contests. Some of
these competitions even tested embodiment and robotic capacities,
while combining multiple tasks. For example, the DARPA {\em Robotics
Challenge} tested a robot's ability to conduct tasks relevant to
remote operation including turning valves, using a tool to break
through a concrete panel, opening doors, remove debris blocking
entryways.

Unfortunately, the {\em Turing Test} variants as well as the
task-based challenges are not sufficient as true measures of autonomy
in the real-world. Autonomy requires a multi-modal ability and an
integrated embodied system to interact with the environment, and
achieve goals while solving open-world problems with the limited
resources available. None of these tests are interested in measuring
this sort of \emph{intelligence} and \emph{capability}, the sort that
is most relevant from a practical standpoint.

\section{The MacGyver Evaluation Framework}

The proposed evaluation framework, based on the idea of MacGyver-esque
creativity, is intended to answer the question whether embodied
machines can {\em generate, execute and learn strategies for
identifying and solving seemingly-unsolvable real-world problems}.
The idea is to present an agent with a problem that is unsolvable with
the agent's initial knowledge and observing the agent's problem
solving processes to estimate the probability that the agent is being
creative: if the agent can think outside of its current context, take
some exploratory actions, and incorporate relevant environmental cues
and learned knowledge to make the problem tractable (or at least
computable) then the agent has the general ability to solve open-world
problems more effectively.\footnote{Note that the proposed MT is a
subset of Harnad's T3, but instead of requiring robots to do
``everything real people do'', MT is focused on requiring robots to
exhibit resourcefulness and resilience.  MT is also a subset of
Schweizer's T4 which expands T3 with the notion of species-level
intelligence.}
%Whether or not MT tests for intentionality, consciousness or sentience are very interesting philosophical questions. However, we do not explore this question in this paper. The question is difficult enough as is for us, and we instead focus on the question of autonomy, which we think is a more pressing question. 

This type of problem solving framework is typically used in the area
of automated planning for describing various sorts of problems and
solution plans and is naturally suited for defining a
MacGyver-esque problem and a creative solution strategy.  We are now
ready to formalize various notions of the MacGyver evaluation
framework.

\subsection{Preliminaries - Classical Planning}

We define $\mathcal{L}$ to be a first order language with predicates
$p(t_1, \ldots , t_n)$ and their negations $\lnot p(t_1, \ldots ,
t_n)$ , where $t_i$ represents terms that can be variables or
constants. A predicate is grounded if and only if all of its terms are
constants. We will use classical planning notions of a planning domain
in $\mathcal{L}$ that can be represented as $\Sigma = (S, A, \gamma
)$, where $S$ represents the set of states, $A$ is the set of actions
and $\gamma$ are the transition functions. A classical planning
problem is a triple $\mathcal{P}=(\Sigma , s_0, g)$, where $s_0$ is
the initial state and $g$ is the goal state. A plan $\pi$ is any
sequence of actions and a plan $\pi$ is a solution to the planning problem if $g \subseteq \gamma (s_0,\pi )$. We also consider the notion
of state reachability and the set of all successor states
$\hat{\Gamma}(s)$, which defines the set of states reachable from $s$.

\subsection{A MacGyver Problem}

To formalize a MacGyver Problem (MGP), we define a universe and then a
world within this universe. The world describes the full set of
abilities of an agent and includes those abilities that the agent
knows about and those of which it is unaware. We can then define an
agent subdomain as representing a proper subset of the world that is
within the awareness of the agent. An MGP then becomes a planning
problem defined in the world, but outside the agent's current
subdomain.

\begin{definition}
\emph{(Universe).}
\label{def:universe}
We first define a Universe $\mathbb{U} = (S,A,\gamma)$ as a classical
planning domain representing all aspects of the physical world
perceivable and actionable by any and all agents, regardless of
capabilities. This includes all the allowable states, actions and
transitions in the physical universe.
 \end{definition}

\begin{definition}
\emph{(World).}
\label{def:world}
We define a world $\mathbb{W}^t = (S^t, A^t, \gamma^t)$ as a portion
of the Universe $\mathbb{U}$ corresponding to those aspects that are
perceivable and actionable by a particular species $t$ of agent. Each
agent species $t \in T$ has a particular set of sensors and actuators
allowing agents in that species to perceive a proper subset of states,
actions or transition functions. Thus, a world can be defined as
follows:
\begin{multline*}
\mathbb{W}^t = \{(S^t, A^t, \gamma^t) \mid ((S^t \subseteq S) \lor (A^t \subseteq A) \lor (\gamma^t \subseteq \gamma)) \\ \land \lnot ((S^t = S) \land (A^t = A) \land (\gamma^t = \gamma))\}
\end{multline*}
\end{definition}

\begin{definition}
\emph{(Agent Subdomain).}
\label{def:agent}
We next define an agent $\Sigma_i^t = (S_i^t, A_i^t, \gamma_i^t)$ of
type $t$, as a planning subdomain corresponding to the agent's perception
and action within its world $\mathbb{W}^t$. In other words, the agent
is not fully aware of all of its capabilities at all times, and the
agent domain $\Sigma_i^t$ corresponds to the portion of the world that
the agent is perceiving and acting at time
$i$.  \begin{multline*} \Sigma_i^t = \{(S_i^t, A_i^t, \gamma_i^t) \mid
((S_i^t \subset S^t) \lor (A_i^t \subset A^t) \lor
(\gamma_i^t \subset \gamma^t)) \\ \land \lnot ((S_i^t = S^t) \land
(A_i^t = A^t) \land (\gamma_i^t = \gamma^t))\} \end{multline*}
	
\end{definition}

\begin{definition}
	\emph{(MacGyver Problem).}
	\label{def:MGP}
	We define a MacGyver Problem (MGP) with respect to an agent $t$, as a planning problem in the agent's world $\mathbb{W}_t$ that has a goal state $g$ that is currently unreachable by the agent. Formally, an MGP $\mathcal{P}_M = (\mathbb{W}^t, s_0, g)$, where:
	\begin{itemize}
		\item $s_0 \in S^t_i$ is the initial state of the agent
		\item $g$ is a set of ground predicates
		\item $S_g = \{s \in S | g \subseteq s \}$ 
	\end{itemize}
	Where $g \subseteq s^\prime, \forall s^\prime \in \hat{\Gamma}_{\mathbb{W}^t}(s_0) \setminus \hat{\Gamma}_{\Sigma^t_i}(s_0)$

\end{definition}

It naturally follows that in the context of a world $\mathbb{W}_t$,
the MGP $\mathcal{P}_M$ is a classical planning problem which from the
agent's current perspective is unsolvable. We can reformulate the MGP
as a language recognition problem to be able to do a brief complexity
analysis.

\begin{definition}
\emph{(MGP-EXISTENCE).}
Given a set of statements $D$ of planning problems,
let \emph{MGP-EXISTENCE($D$)} be the set of all statements $P \in D$
such that $P$ represents a MacGyver Problem $\mathcal{P}_M$, without
any syntactical restrictions.
\end{definition}

\begin{theorem}
\label{thm:decidable}
MGP-EXISTENCE is decidable.
\end{theorem}
\begin{proof}
The proof is simple. The number of possible states in the agent's
subdomain $\Sigma_i^t$ and the agent's world $\mathbb{W}^t$ are
finite. So, it is possible to do a brute-force search to see whether a
solution exists in the agent's world but not in the agent's initial
domain.
\end{proof}

\begin{theorem}
\label{thm:expspace}
MGP-EXISTENCE is EXPSPACE-\textit{complete}.	
\end{theorem}
\begin{proof}
\textit(Membership). An MGP amounts to looking to see if the problem is a
solvable problem in the agent-domain. Upon concluding it is not
solvable, the problem then becomes one of determining if it is a
solvable problem in the world corresponding to the agent's
species. Each of these problems are PLAN-EXISTENCE problems, which are
in EXPSPACE for the unrestricted case \cite{Ghallab2004}. Thus,
MGP-EXISTENCE is in EXPSPACE.

\textit(Hardness). We can reduce the classical planning
problem $P(\Sigma,s_0,g)$ to an MGP (PLAN-EXISTENCE $\leq_m^p$
MGP-EXISTENCE), by defining a new world $\mathbb{W}$. To define a new
world, we extend the classical domain by one state, defining the new
state as a goal state, and adding actions and transitions from every
state to the new goal state. We also set the agent domain to be the
same as the classical planning domain. Now, $P(\Sigma,s_0,g) \in $
PLAN-EXISTENCE iff $P(\mathbb{W},s_0,g) \in $ MGP-EXISTENCE for agent
with domain $\Sigma$. Thus, MGP-EXISTENCE is
EXPSPACE-\textit{hard}.
\end{proof}

\subsection{Solving a MacGyver Problem}

From Theorems~\ref{thm:decidable} and \ref{thm:expspace}, we know
that, while possible, it is intractable for an agent to know whether a
given problem is an MGP. From an agent's perspective, solving an MGP
is like solving any planning problem with the additional requirement
to sense or learn some previously unknown state, transition function
or action. Specifically, solving an MGP will involve performing some
actions in the environment, making observations, extending and
contracting the agent's subdomain and exploring different contexts.

%\begin{definition}{(MGP Solution).}
%	A plan $\pi$ is a solution for $\mathcal{P}_M$ if $g \subseteq \gamma (s_0,\pi )$.
%\end{definition}

\subsubsection{Solution Strategies}

\begin{definition}
\emph{(Agent Domain Modification).}  \label{def:domainExt} A
domain modification $\Sigma^{t*}_j$ involves either a domain extension
or a domain contraction\footnote{In the interest of brevity we will
only consider domain extensions for now.}. A domain extension
$\Sigma^{t+}_j$ of an agent is an Agent-subdomain at time $j$ that is
in the agent's world $\mathbb{W}^t$ but not in the agent's subdomain
$\Sigma^t_i$ in the previous time $i$, such that
$\Sigma^t_i \preceq \Sigma^t_j$. The agent extends its subdomain
through sensing and perceiving its environment and its own self -
e.g., the agent can extend its domain by making an observation, receiving
advice or an instruction or performing
introspection. Formally, \begin{multline*} \Sigma^{t+}_j
= \{(S_j^{t+}, A_j^{t+}, \gamma_j^{t+}) \mid (S_j^{t+} \subset
S^t \setminus S_i^t) \\ \lor (A_j^{t+} \subset A^t \setminus
A_i^t) \lor
(\gamma_j^{t+} \subset \gamma^t \setminus \gamma_i^t)\} \end{multline*}
	
%	A domain contraction $\Sigma^{t-}_j$ of an agent is an Agent-subdomain at time $j$ that is in the agent's subdomain $\Sigma^t_i$ in the previous time $i$, such that $\Sigma^t_i \preceq \Sigma^t_j$. The agent contracts its subdomain by discarding domain information not relevant in the current time step. Formally,
%	\begin{multline*}
%	\Sigma^{t-}_j = \{(S_j^{t-}, A_j^{t-}, \gamma_j^{t-}) \mid (S_j^{t-} \subset S_i^t) \lor (A_j^{t-} \subset  A_i^t) \\ \lor (\gamma_j^{t-} \subset  \gamma_i^t)\}
%	\end{multline*}
	
The agent subdomain that results from a domain extension is
$\Sigma_j^t = \Sigma_i^t \cup \Sigma^{t+}_j$
%	and from a domain contraction is $\Sigma_j^t = \Sigma_i^t \setminus \Sigma^{t-}_j$

A domain modification set
$\Delta_{\Sigma^t_i}= \{ \Sigma^{t*}_1, \Sigma^{t*}_2, \ldots
,\Sigma^{t*}_n \}$ is a set of $n$ domain modifications on subdomain
$\Sigma^t_i$. Let $\Sigma_{\Delta}^{t}$ be the subdomain resulting
from applying $\Delta_{\Sigma^t_i}$ on $\Sigma^t_i$
	
\end{definition}

\begin{definition}
\emph{(Strategy and Domain-Modifying Strategy).}
A {\em strategy} is a tuple $\omega = (\pi, \Delta)$ of a plan $\pi$ and a
set $\Delta$ of domain modifications. A domain-modifying strategy
$\omega^C$ involves at least one domain modification, i.e.,
$\Delta \neq \emptyset$.
\end{definition}

\begin{definition}
\emph{(Context).}
A {\em context} is a tuple $\mathbb{C}_i = (\Sigma_i,s_i)$
representing the agent's subdomain and state at time $i$.
\end{definition}

We are now ready to define an insightful strategy as a set of actions
and domain modifications that the agent needs to perform to allow for
the goal state of the problem to be reachable by the agent.

\begin{definition}
\emph{(Insightful Strategy).}
Let $\mathbb{C}_i = (\Sigma^t_i, s_0)$ be the agent's current
context. Let $\mathcal{P}_M=(\mathbb{W}^t, s_0, g)$ be an MGP for the
agent in this context. An insightful strategy is a domain-modifying
strategy $\omega^I= (\pi^I, \Delta^I)$ which when applied in
$\mathbb{C}_i$ results in a context $\mathbb{C}_j= (\Sigma^t_j, s_j)$,
where $\Sigma^t_j = \Sigma^t_{\Delta^I}$ such that $g \subseteq
s^\prime, \forall s^\prime \in \hat{\Gamma}_{\Sigma_j^t}(s_j)$.
\end{definition}

Formalizing the insightful strategy in this way is somewhat analogous
to the moment of insight that is reached when a problem becomes
tractable (or in our definition computable) or when solution plan
becomes feasible. Specifically, solving a problem involves some amount
of creative exploration and domain extensions and contractions until
the point when the agent has all the information it needs within its
subdomain to solve the problem as a classical planning problem, and
does not need any further domain extensions. We can alternatively
define an insightful strategy in terms of when the goal state is not
only reachable, but a solution can be discovered in polynomial
time. We will next review simple toy examples to illustrate the
concepts discussed thus far.

\section{Operationalizing a MacGyver Problem}

We will consider two examples that will help operationalize the
formalism presented thus far. The first is a modification of the
popular {\em Blocks World} planning problem. The second is a more
practical task of tightening screws, however, with the caveat that
certain common tools are unavailable and the problem solver must
improvise. We specifically discuss various capabilities that an agent
must possess in order to overcome the challenges posed by the examples.

\subsection{Toy Example: Block-and-Towel World}

Consider an agent tasked with moving a block from one location to
another which the agent will not be able to execute without first
discovering some new domain information.  Let the agent subdomain
$\Sigma$ consist of a set of locations $l = \{L1, L2, L3\}$, two
objects $o = \{T,B\}$ a towel and a block, and a function
$\mathit{locationOf}:o \rightarrow l$ representing the location of
object $o$. Suppose the agent is aware of the following predicates and
their negations:

\begin{itemize}
\itemsep-0.2em
\item $at(o,l)$: an object $o$ is at location $l$
\item $near(l)$: the agent is near location $l$
\item $touching(o)$: the agent is touching object $o$
\item $holding(o)$: the agent is holding the object $o$
\end{itemize}

\noindent We define a set of actions $A$ in the agent domain as follows:
	
\begin{itemize} \itemsep-0.2em \item $reach(o,l)$: Move the
		robot arm to near the object $o$ \\ precond:
		$\{at(o,l)\}$ \\ effect: $\{near(l)\}$ \item
		$grasp(o,l)$: Grasp object $o$ at $l$ \\ precond:
		$\{near(l),at(o,l)\}$ \\ effect:
		$\{touching(o)\}$ \item $\mathit{lift}(o,l)$: Lift
		object $o$ from $l$ \\ precond:
		$\{touching(o),at(o,l)\}$ \\ effect:
		$\{holding(o),\lnot at(o,l),\lnot near(l)\}$ \item
		$carryTo(o,l)$: Carry object $o$ to $l$ \\ precond:
		$\{holding(o)\}$ \\ effect: $\{\lnot holding(o),
		at(o,l)\}$ \item $release(o,l)$: Release object $o$ at
		$l$ \\ precond: $\{touching(o),at(o,l)\}$ \\ effect:
		$\{\lnot touching(o), at(o,l)\}$ \end{itemize}

Given an agent domain $\Sigma$, and a start state $s_0$ as defined
below, we can define the agent context $\mathbb{C} = (\Sigma, s_0)$ as
a tuple with the agent domain and the start state.
\begin{multline*}
s_0 = \{at(T,L1), at(B,L2), \lnot holding(T), \lnot
holding(B), \\ \lnot near(L1), \lnot near(L2), \lnot near(L3), \lnot
touching(T), \\ \lnot touching
(B), \mathit{locationOf}(B)=L2, \\ \mathit{locationOf}(T)=L1 \}
\end{multline*}

Consider a simple planning problem for the Block-and-Towel World in
which the agent must move the block $B$ from location $L2$ to
$L3$. The agent could execute a simple plan as follows to solve the
problem:
\begin{multline*}
\pi_1 = \{reach(B,\mathit{locationOf}(B)), \\ grasp(B,\mathit{locationOf}(B)), \\ lift(B,\mathit{locationOf}(B)), carryTo(B,L3), release(B,L3)\}
\end{multline*}

%Upon executing the plan, the agent is in a final state:
%
%$$s_{\pi_1} = s_0 \setminus \{at(B,L2)\} \cup \{at(B,L3)\}$$

During the course of the plan, the agent touches and holds the block
as it moves it from location $L2$ to $L3$. Using a similar plan, the
agent could move the towel to any location, as well.

Now, consider a more difficult planning problem in which the agent is
asked to move the block from $L2$ to $L3$ without touching it. Given
the constraints imposed by the problem, the goal state, is not
reachable and the agent must discover an alternative way to move the
block. To do so, the agent must uncover states in the world
$\mathbb{W}^t$ that were previously not in its subdomain $\Sigma$. For
example, the agent may learn that by moving the towel to location
$L2$, the towel ``covers" the block, so it might discover a new
predicate $covered(o1,o2)$ that would prevent it from touching the
block. The agent may also uncover a new action
$\mathit{push}(o,l1,l2)$ which would allow it to push the object along
the surface. To uncover new predicates and actions, the agent may have
to execute an insightful strategy $\omega^I$. Once the agent's domain
has been extended, the problem becomes a standard planning problem for
which the agent can discover a solution plan for covering the block
with the towel and then pushing both the towel and the block from
location $L2$ to $L3$. In order to autonomously resolve this problem,
the agent must be able to recognize when it is stuck, discover new
insights, and build new plans. Additionally, the agent must be able to
actually execute this operation in the real-world. That is, the agent
must have suitable robotic sensory and action mechanisms to locate and
grasp and manipulate the objects.

\subsection{Practical Example: Makeshift Screwdriver}

Consider the practical example of attaching or fastening things
together, a critical task in many domains, which, depending on the
situation, can require resilience to unexpected events and
resourcefulness in finding solutions.  Suppose an agent must affix two
blocks from a set of blocks $b = \{B1,B2\}$. In order to do so, the
agent has a tool box containing a set of tools $t = \{screwdriver,
plier, hammer\}$ and a set of fasteners $f= \{screw,nail\}$. In
addition, there are other objects in the agent's environment $o
=\{towel,coin,mug,duct tape\}$. Assume the agent can sense the following
relations (i.e., predicates and their negations) with respect to the
objects \footnote{Given space limitations, we have not presented the
entire domain represented by this example. Nevertheless, our analysis
of the MacGyver-esque properties should still hold.}:

\begin{itemize}
	\itemsep-0.2em 
	\item $isAvailable(t)$: tool $t$ is available to use 
	\item $\mathit{fastenWith}(t,f)$: tool $t$ is designed for fastener $f$
	\item $grabWith(t)$: tool $t$ is designed to grab a fastener $f$
	\item $isHolding(t)$: agent is holding tool $t$
	\item $isReachable(t,f)$: tool $t$ can reach fastener $f$
	\item $isCoupled(t,f)$: tool $t$ is coupled to fastener $f$
	\item $isAttachedTo(f,b1,b2)$: fastener $f$ is attached to or inserted into blocks $b1$ and $b2$
	\item $isSecured(f,b1,b2)$: fastener $f$ is tightly secured into blocks $b1$ and $b2$.
\end{itemize}

We can also define a set of actions in the agent subdomain as follows:

\begin{itemize}
\itemsep-0.2em 
	\item $select(t,f)$: select/grasp a tool $t$ to use with fastener $f$ \\
		precond: $\{isAvailable(t), fastenWith(t,f)\}$\\
		effect: $\{isHolding(t)\}$
	\item $grab(t,f)$: Grab a fastener $f$ with tool $t$.\\
		precond: $\{isHolding(t),grabWith(t)\}$\\
		effect: $\{isCoupled(t,f)\}$
	\item $placeAndAlign(f,b1,b2)$: Place and align fastener $f$, and blocks $b1$ and $b2$\\
		effect: $\{isAttachedTo(f,b1,b2)\}$
	\item $reachAndEngage(t,f)$: Reach and engage the tool $t$ with fastener $f$ \\
		precond: $\{isHolding(t),fastenWith(t,f),$\\$isReachable(t,f)\}$\\
		effect: $\{isCoupled(t,f)\}$
	\item $install(f,t,b1,b2)$: Install the fastener $f$ with tool $t$\\
		precond: $\{isCoupled(t,f),isAttachedTo(f,b1,b2)\}$\\
		effect: $\{isSecured(f,b1,b2\}$
\end{itemize}

Now suppose a screw has been loosely inserted into two blocks
($isAttachedTo(screw,B1,B2)$) and needs to be tightened ($\lnot
isSecured(screw,B1,B2)$). Tightening a screw would be quite
straightforward by performing actions $select(),reachAndEngage(),
install()$. But for some reason the screwdriver has gone missing
($\lnot isAvailable(screwdriver)$).

This is a MacGyver problem because there is no way for the agent,
given its current subdomain of knowledge, to tighten the screw as the
goal state of $isSecured(screw,B1,B2)$ is unreachable from the agent's
current context. Hence, the agent must extend its domain. One approach
is to consider one of the non-tool objects, e.g., a coin could be used
as a screwdriver as well, while a mug or towel might not.
%The agent, however, does not know which tools or objects will and will not work. In fact, the agent does not even know that a non-tool object can even be used as a tool. The agent must hypothesize some strategies and then experiment with the objects available in order to extend its domain.

The agent must be able to switch around variables in its existing
knowledge to expose previously unknown capabilities of tools. For
example, by switching $grab(t,f)$ to $grab(t,o)$ the agent can now
explore the possibility of grabbing a coin with a plier. Similarly, by
relaxing constraints on variables in other relations, the agent can
perform a $reachAndEngage(o,f)$ action whereby it can couple a
makeshift tool, namely the coin, with the screw.

What if the screw was in a recessed location and therefore difficult
to access without an elongate arm? While the coin might fit on the
head of the screw, it does not have the necessary elongation and would
not be able to reach the screw. An approach here might be to grab the
coin with the plier and use that assembly to tighten the screw, maybe
even with some additionally duct tape for extra support. As noted
earlier, generally, the agent must be able to relax some of the
pre-existing constraints and generate new actions. By exploring and
hypothesizing and then testing each variation, the agent can expand
its domain.

This example, while still relatively simple for humans, helps us
highlight the complexity of resources needed for an agent to perform
the task. Successfully identifying and building a makeshift
screwdriver when a standard screwdriver is not available shows a
degree of resilience to events and autonomy and resourcefulness that
we believe to be an important component of everyday creativity and
intelligence. By formulating the notion of resourcefulness in this
manner, we can better study the complexity of the cognitive processes
and also computationalize these abilities and even formally measure
them.

\subsubsection{Agent Requirements: Intelligence and Physical Embodiment}

When humans solve problems, particularly creative insight problems,
they tend to use various heuristics to simplify the search space and
to identify invariants in the environment that may or may not be
relevant \cite{knoblich2009psychological}. An agent solving an MGP
must possess the ability to execute these types of heuristics and
cognitive strategies. Moreover, MGPs are not merely problems in the
classical planning sense, but require the ability to discover when a
problem is unsolvable from a planning standpoint and then discover,
through environmental exploration, relevant aspects of its
surroundings in order to extend its domain of knowledge. Both these
discoveries in turn are likely to require additional cognitive
resources and heuristics that allow the agent to make these
discoveries efficiently. Finally, the agent must also be able to remember
this knowledge and be able to, more efficiently, solve future
instances of similar problems.

From a real-world capabilities standpoint, the agent must possess the
sensory and action capabilities to be able to execute this exploration
and discovery process, including grasping and manipulating unfamiliar
objects. These practical capabilities are not trivial, but in
combination with intelligent reasoning, will provide a clear
demonstration of agent autonomy while solving practical real-world
problems.

These examples provide a sense for the types of planning problems that
might qualify as an MGP. Certain MGPs are more challenging than others
and we will next present a theoretical measure for the difficulty of
an MGP.

\subsection{Optimal Solution and M-Number}

Generally, we can assume that a solvable MGP has a best solution that
involves an agent taking the most effective actions, making the
required observations as and when needed and uncovering a solution
using the most elegant strategy. We formalize these notions by first
defining optimal solutions and then the M-Number, which is the measure
of the complexity of an insightful strategy in the optimal solution.

%\begin{definition}
%	\emph{(Non-redundant Plan).}
%	A non-redundant plan is a plan that does not transition through a state more than once. 
	
%	That is, $\pi = \langle a_1, \ldots , a_n \rangle$ is  non-redundant plan in which $\gamma(s_m,a_i) \neq \gamma(s_p,a_j), \forall a_i, a_j \in \pi, i \neq j$.  
%\end{definition}

\begin{definition}
	\emph{(Optimal Solutions).}
	Let $\mathcal{P}_M=(\mathbb{W}^t, s_0, g)$ be an MGP for the agent. Let $\hat{\pi}$ be an optimal solution plan to $\mathcal{P}_M$. 	
	A set of optimal domain modifications is a set of domain modifications $\hat{\Delta}$ is the minimum set of domain modifications needed for the inclusion of actions in the optimal solution plan $\hat{\pi}$.
	An optimal solution strategy is a solution strategy $\hat{\omega} = (\hat{\pi},\hat{\Delta})$, where $\hat{\Delta}$ is a set of optimal domain modifications. 
\end{definition}

\begin{definition}
	\emph{(M-Number).}
	Let $\mathcal{P}_M=(\mathbb{W}^t, s_0, g)$ be an MGP for the agent. Let $\hat{\Omega} = \{\hat{\omega}_1, \ldots, \hat{\omega}_n\}$ be the set of $n$ optimal solution strategies. For each $\hat{\omega}_i \in \hat{\Omega}$, there exists an insightful strategy $\hat{\omega}_i^I \subseteq \hat{\omega}_i$. Let $\hat{\Omega}^I = \{\hat{\omega}_1^I, \ldots, \hat{\omega}_n^I\}$ be the set of optimal insightful strategies. The set $\hat{\Omega}^I$ can be represented by a program $p$ on some prefix universal Turing machine capable of listing elements of $\hat{\Omega}^I$ and halting. We can then use Kolmogorov complexity of the set of these insightful strategies, $K(\hat{\Omega}^I) := \min_{p \in \mathbb{B}^*} \{|p| : \mathcal{U}(p) \mbox{ computes } \hat{\Omega}^I\}$ \cite{Li1997}. We define the intrinsic difficulty of the MGP (M-Number or $\mathcal{M}$) as the Kolmogorov complexity of the set of optimal insightful strategies $\hat{\Omega}^I$, $\mathcal{M} = K(\hat{\Omega}^I)$.
\end{definition}

As we have shown MGP-EXISTENCE is intractable and measuring the
intrinsic difficulty of an MGP is not computable if we use Kolmogorov
complexity. Even if we instead choose to use an alternative and
computable approximation to Kolmogorov complexity (e.g., Normalized
Compression Distance), determining the M-Number is difficult to do as
we must consult an oracle to determine the optimal solution. In
reality, an agent does not know that the problem it is facing is an
MGP and even if it did know this, the agent would have a tough time
determining how well it is doing.

\subsection{Measuring Progress and Agent Success} 

When we challenge each other with creative problems,we often know if
the problem-solver is getting closer (``warmer") to the solution. We
formalize this idea using Solomonoff Induction. To do so, we will
first designate a ``judge" who, based on a strategy currently executed
by the agent, guesses the probability that, in some finite number of
steps, the agent is likely to have completed an insightful strategy.

Consider an agent performing a strategy $\omega$ to attempt to solve
an MGP $\mathcal{P}_M$ and a judge evaluating the performance of the
agent. The judge must first attempt to understand what the agent is
trying to do. Thus, the judge must first hypothesize an agent model
that is capable of generating $\omega$.

Let the agent be defined by the probability measure $\mu
(\omega \mid \mathcal{P}_M, \mathbb{C})$, where this measure
represents the probability that an agent generates a strategy $\omega$
given an MGP $\mathcal{P}_M$ when in a particular context
$\mathbb{C}$. The judge does not know $\mu$ in advance and the measure
could change depending on the type of agent. For example, a random
agent could have $\mu(\omega) = 2 ^{-|\omega|}$, whereas a MacGyver
agent could be represented by a different probability measure. Not
knowing the type of agent, we want the judge to be able to evaluate as
many different types of agents as possible. There are infinitely many
different types of agents and accordingly infinitely many different
hypotheses $\mu$ for an agent. Thus, we cannot simply take an expected
value with respect to a uniform distribution, as some hypotheses must
be weighed more heavily than others.

Solomonoff devised a universal distribution over a set of computable
hypotheses from the perspective of computability
theory \cite{solomonoff1960preliminary}. The universal prior of a
hypothesis was defined:

$$P(\mu) \equiv \sum\limits_{p: \mathcal{U}(p,\omega)=\mu(\omega)} 2^{-|p|}$$

The judge applies the principle of Occam's razor - given many
explanations (in our case hypotheses), the simplest is the most
likely, and we can approximate $P(\mu) \equiv 2^{-K(\mu)}$, where
$K(\mu)$ is the Kolmogorov complexity of measure $\mu$.
%
%Thus, formally, our \textit{a priori} distribution of measures $\mu$ should be weighted towards simpler descriptions of agent behavior. As each agent behavior (i.e., strategies) is computable, it can be represented by a program $p \in \mathbb{B}^*$ on some prefix universal Turing machine $\mathcal{U}$. Thus, the Kolmogorov complexity to measure complexity of $\mu$ is given by $K(\mu) := \min_{p \in \mathbb{B}^*} \{|p| : \mathcal{U}(p) \mbox{ computes } \mu\}$

%The judge must then value this strategy as being suitably resourceful, creative or insightful so that it increases the probability that the agent will solve the MGP.
To be able to measure the progress of an agent solving an MGP, we must
be able to define a performance metric $R_\mu$. In this paper, we do
not develop any particular performance metric, but suggest that a
performance metric be proportional to the level of resourcefulness and
creativity of the agent. Generally, measuring progress may depend on
problem scope, control variables, length and elegance of the solution
and other factors. Nevertheless, a simple measure of this sort can
serve as a placeholder to develop our theory.

We are now ready to define the performance or progress of an agent
solving an MGP.

\begin{definition}
\emph{(Expected Progress).}
Consider an agent in context $\mathbb{C} = (\Sigma_i^t, s_0)$ solving
an MGP $\mathcal{P}_M = (\mathbb{W}^t, s_0, g)$. The agent has
executed strategy $\omega$ comprising actions and domain modifications. Let
$U$ be the space of all programs that compute a measure of agent
resourcefulness. Consider a judge observing the agent and fully aware
of the agent's context and knowledge and the MGP itself. Let the judge
be prefix universal Turing machine $\mathcal{U}$ and let $K$ be the
Kolmogorov complexity function. Let the performance metric, which is
an interpretation of the cumulative state of the agent resourcefulness
in solving the MGP, be $R_\mu$. The expected progress of this agent as
adjudicated by the judge is:
$$M(\omega) \equiv \sum\limits_{\mu \in U}2^{-K(\mu)}\cdot R_\mu  $$
\end{definition}

Now, we are also interested in seeing whether the agent, given this
strategy $\omega$ is likely to improve its performance over the next
$k$ actions. The judge will need to predict the continuation of this
agent's strategy taking all possible hypotheses of the agent's
behavior into account. Let $\omega^+$ be a possible continuation and
let $^\smallfrown$ represent concatenation.

$$M(\omega^\smallfrown \omega^+ \mid \omega ) = \frac{M(\omega^\smallfrown \omega^+)}{M(\omega)}$$

The judge is a Solomonoff predictor such that the predicted finite
continuation $\omega^+$ is likely to be one in which
$\omega^\smallfrown \omega^+$ is less complex in the Kolmogorov
sense. The judge measures the state of the agent's attempts at solving
the MGP and can also predict how the agent is likely to perform in the
future.
%Traditionally, the Solomonoff Induction framework was used predict how an agent would act in an environment. Legg and Hutter formalized this traditional notion and introduced the notion of a reinforcement learning agent deriving rewards from the environment \cite{Legg2005,Legg2006}. Here, we turn this traditional setup around and instead focus on the judge-agent setup since we are evaluating the agent's progress. 

\section{Conclusion and Future Work}

In the Apollo 13 space mission, astronauts together with ground
control had to overcome several challenges to bring the team safely
back to Earth \cite{lovell2006apollo}. One of these challenges was
controlling carbon dioxide levels onboard the space craft: ``For two
days straight [they] had worked on how to jury-rig the Odyssey’s
canisters to the Aquarius’s life support system. Now, using materials
known to be available onboard the spacecraft -- a sock, a plastic bag,
the cover of a flight manual, lots of duct tape, and so on -- the crew
assembled a strange contraption and taped it into place. Carbon
dioxide levels immediately began to fall into the safe
range.'' \cite{cass2005apollo,apollomission}.

We proposed the {\em MacGyver Test} as a practical alternative to the
{\em Turing} Test and as a formal alternative to robotic and machine
learning challenges.  The MT does not require any specific internal
mechanism for the agent, but instead focuses on observed
problem-solving behavior akin to the Apollo 13 team. It is flexible
and dynamic allowing for measuring a wide range of agents across
various types of problems. It is based on fundamental notions of set
theory, automated planning, Turing computation, and complexity theory
that allow for a formal measure of task difficulty. Although
Kolmogorov complexity and the Solomonoff Induction measures are not
computable, they are formally rigorous and can be substituted with
computable approximations for practical applications.

In future work, we plan to develop more examples of MGPs and also
begin to unpack any interesting aspects of the problem's structure,
study its complexity and draw comparisons between problems. We believe
that the MT formally captures the concept of practical intelligence
and everyday creativity that is quintessentially human and practically
helpful when designing autonomous agents. Most importantly, the intent
of the MT and the accompanying MGP formalism is to help guide research
by providing a set of mathematically formal specifications for
measuring AI progress based on an agent's ability to solve
increasingly difficult MGPs. We thus invite researchers to develop
MGPs of varying difficulty and design agents that can solve them.

\bibliographystyle{named}
\bibliography{references}

\begin{thebibliography}{}

\bibitem[\protect\citeauthoryear{Boden}{2010}]{Boden2010}
Margaret~A. Boden.
\newblock {The Turing test and artistic creativity}.
\newblock {\em Kybernetes}, 39(3):409--413, 2010.

\bibitem[\protect\citeauthoryear{Bringsjord and Sen}{2016}]{Bringsjord2016}
Selmer Bringsjord and Atriya Sen.
\newblock {On Creative Self-Driving Cars: Hire the Computational Logicians,
  Fast}.
\newblock page Forthcoming, 2016.

\bibitem[\protect\citeauthoryear{Bringsjord \bgroup \em et al.\egroup
  }{2001}]{Bringsjord2001}
Selmer Bringsjord, Paul Bello, and David Ferrucci.
\newblock {Creativity, the Turing test, and the (better) Lovelace Test}.
\newblock {\em Minds and Machines}, 11(1):3--27, 2001.

\bibitem[\protect\citeauthoryear{Cass}{2005}]{cass2005apollo}
Stephen Cass.
\newblock Apollo 13, we have a solution.
\newblock {\em IEEE Spectrum On-line, 04}, 1, 2005.

\bibitem[\protect\citeauthoryear{Cohen}{2005}]{Cohen2005}
Paul~R Cohen.
\newblock If not turing's test, then what?
\newblock {\em AI Magazine}, 26(4):61, 2005.

\bibitem[\protect\citeauthoryear{Cooper and Van~Leeuwen}{2013}]{cooper2013alan}
S~Barry Cooper and Jan Van~Leeuwen.
\newblock {\em Alan Turing: His work and impact}.
\newblock 2013.

\bibitem[\protect\citeauthoryear{Feigenbaum}{2003}]{Feigenbaum2003}
Edward~A. Feigenbaum.
\newblock {Some challenges and grand challenges for computational
  intelligence}.
\newblock {\em Journal of the ACM}, 50(1):32--40, 2003.

\bibitem[\protect\citeauthoryear{Ghallab \bgroup \em et al.\egroup
  }{2004}]{Ghallab2004}
Malik Ghallab, Dana Nau, and Paolo Traverso.
\newblock {\em {Automated Planning: Theory and Practice}}.
\newblock 2004.

\bibitem[\protect\citeauthoryear{Harnad}{1991}]{Harnad1991}
Stevan Harnad.
\newblock {Other bodies, other minds: A machine incarnation of an old
  philosophical problem}.
\newblock {\em Minds and Machines}, 1(1):43--54, 1991.

\bibitem[\protect\citeauthoryear{Knoblich}{2009}]{knoblich2009psychological}
G{\"u}nther Knoblich.
\newblock Psychological research on insight problem solving.
\newblock In {\em Recasting reality}, pages 275--300. Springer, 2009.

\bibitem[\protect\citeauthoryear{Levesque \bgroup \em et al.\egroup
  }{2012}]{Levesque2012}
Hector Levesque, Ernest Davis, and Leora Morgenstern.
\newblock {The Winograd Schema Challenge}.
\newblock In {\em Thirteenth International Conference on the Principles of
  Knowledge Representation and Reasoning}, 2012.

\bibitem[\protect\citeauthoryear{Li and Vit{\'{a}}nyi}{1997}]{Li1997}
M~Li and Paul Vit{\'{a}}nyi.
\newblock {An introduction to Kolmogorov complexity and its applications.
  (Second edition)}.
\newblock {\em Computers {\&} Mathematics with Applications}, 34:137, 1997.

\bibitem[\protect\citeauthoryear{Lovell and Kluger}{2006}]{lovell2006apollo}
Jim Lovell and Jeffrey Kluger.
\newblock {\em Apollo 13}.
\newblock Houghton Mifflin Harcourt, 2006.

\bibitem[\protect\citeauthoryear{Riedl}{2014}]{Riedl2014}
Mark~O. Riedl.
\newblock {The Lovelace 2.0 Test of Artificial Creativity and Intelligence}.
\newblock {\em arXiv preprint arXiv:1410.6142v3}, page~2, 2014.

\bibitem[\protect\citeauthoryear{Schweizer}{2012}]{Schweizer2012}
Paul Schweizer.
\newblock {The externalist foundations of a truly total turing test}.
\newblock {\em Minds and Machines}, 22(3):191--212, 2012.

\bibitem[\protect\citeauthoryear{Solomonoff}{1960}]{solomonoff1960preliminary}
RJ~Solomonoff.
\newblock A preliminary report on a general theory of inductive inference.
\newblock {\em Zator Technical Bulletin}, (138), 1960.

\bibitem[\protect\citeauthoryear{Team}{1970}]{apollomission}
Mission~Evaluation Team.
\newblock Mission operations report apollo 13.
\newblock 1970.

\bibitem[\protect\citeauthoryear{Turing}{1950}]{Turing1950}
Alan~M Turing.
\newblock {Computing Machine and Intelligence}.
\newblock {\em MIND}, LIX(236):433--460, 1950.

\end{thebibliography}

\end{document}